\documentclass[nohyperref]{article}

\usepackage{amsmath,amsfonts,bm}

\def\eqref#1{equation~\ref{#1}}

\def\1{\bm{1}}

\def\kdpp{{$k\textnormal{-DPP}$}}

\DeclareMathAlphabet{\mathsfit}{\encodingdefault}{\sfdefault}{m}{sl}
\SetMathAlphabet{\mathsfit}{bold}{\encodingdefault}{\sfdefault}{bx}{n}

\DeclareMathOperator{\Tr}{Tr}

\allowdisplaybreaks

\usepackage{microtype}
\usepackage{graphicx}
\usepackage{enumerate}
\usepackage{booktabs} 

\usepackage{hyperref}

\usepackage{amsmath}
\usepackage{amssymb}
\usepackage{mathtools}
\usepackage{amsthm}
\usepackage{color}
\usepackage{subfig}
\usepackage{float}
\usepackage{mwe}
\usepackage{cleveref}

\theoremstyle{plain}
\newtheorem{theorem}{Theorem}[section]
\newtheorem{proposition}[theorem]{Proposition}
\newtheorem{lemma}[theorem]{Lemma}
\newtheorem{corollary}[theorem]{Corollary}
\theoremstyle{definition}
\newtheorem{definition}[theorem]{Definition}

\theoremstyle{remark}

\usepackage[textsize=tiny]{todonotes}

\newcommand{\figwidthtwo}{0.48\textwidth}
\newcommand{\figwidththree}{0.32\textwidth}
\newcommand{\figwidthfour}{0.245\textwidth}

\newcommand{\algoname}{DNS}

\usepackage[accepted]{icml2021}

\icmltitlerunning{\algoname: \textbf{\underline{D}}eterminantal Point Process Based \textbf{\underline{N}}eural Network \textbf{\underline{S}}ampler for Ensemble Reinforcement Learning}

\begin{document}

\twocolumn[
\icmltitle{\algoname: \textbf{\underline{D}}eterminantal Point Process Based \textbf{\underline{N}}eural Network \textbf{\underline{S}}ampler\\ for Ensemble Reinforcement Learning}



\icmlsetsymbol{equal}{*}

\begin{icmlauthorlist}
\icmlauthor{Hassam Ullah Sheikh}{equal,labs}
\icmlauthor{Kizza Nandyose Frisbee}{equal,hr}
\icmlauthor{Mariano Phielipp}{labs}
\end{icmlauthorlist}

\icmlaffiliation{labs}{Intel Labs}
\icmlaffiliation{hr}{Intel Corporation}

\icmlcorrespondingauthor{Hassam Sheikh}{hassam.sheikh@intel.com, hassamsheikh1@gmail.com}

\icmlkeywords{Compute-efficient Reinforcement Learning}

\vskip 0.3in
]



\printAffiliationsAndNotice{\icmlEqualContribution} 

\begin{abstract}
The application of an ensemble of neural networks is becoming an imminent tool for advancing state-of-the-art deep reinforcement learning algorithms. However, training these large numbers of neural networks in the ensemble has an exceedingly high computation cost which may become a hindrance in training large-scale systems. In this paper, we propose \algoname: a \textbf{\underline{D}}eterminantal Point Process based \textbf{\underline{N}}eural Network \textbf{\underline{S}}ampler that specifically uses \kdpp\ to sample a subset of neural networks for backpropagation at every training step thus significantly reducing the training time and computation cost. We integrated \algoname\ in REDQ for continuous control tasks and evaluated on MuJoCo environments. Our experiments show that \algoname\ augmented REDQ matches the baseline REDQ in terms of average cumulative reward and achieves this using less than 50\% computation when measured in FLOPS.

\end{abstract}

\section{Introduction}
In the past decade, reinforcement learning (RL) algorithms powered by high-capacity function approximators such as deep neural networks have been used to master complex sequential decision problems such as Atari games~\citep{Mnih-2015-Nature}, boards games like Chess, Go and Shogi~\citep{Silver-2016-Nature,Silver-2017-Nature,Silver-2018-Science} and robotic manipulation~\citep{Liu-2021-MDPI}. Despite having impressive results, deep reinforcement learning (DRL) algorithms suffer from whole host of problems such as sample inefficiency~\citep{Kaiser-2020-ICLR}, overestimation bias~\citep{Hasselt-2010-NIPS,Hasselt-2016-AAAI,Lan-2020-ICLR,Anschel-2017-ICML,fujimoto2018addressing} and imbalance between exploration and exploitation~\citep{lee2020sunrise,osband2016deep}.   

Considering the success of ensembles in supervised learning, use of ensemble of neural networks is becoming popular in deep reinforcement learning (DRL) and are being used to address the aforementioned issues. For example, in~\citep{Lan-2020-ICLR,Anschel-2017-ICML,fujimoto2018addressing} have used ensemble to address the overestimation bias problem. In~\citep{Chen-2021-ICLR} proposed REDQ that uses ensemble with high update-to-date ratio to address the sample inefficiency problem. Similarly~\citep{lee2020sunrise} have used ensemble for efficient exploration. 

Despite ensembles providing elegant theoretical and practical solutions, they introduce new practical problems such as high computation cost and long training times. The high computation cost problem is specially evident in actor-critic settings where DRL algorithms use a high number of critic networks. One such example is the REDQ algorithm that uses ten critic networks and updates all of them in every training step which leads to high computation cost as well as high wall-clock time.

To address this issue, we present \algoname: a \textbf{D}eterminantal Point Process based \textbf{N}eural Network \textbf{S}ampler that specifically uses \kdpp~\citep{Kulesza-2011-ICML} to sample a subset of neural networks for backpropagation at every training step.  \algoname\  uses Centered Kernel Alignment (CKA)~\citep{Kornblith-2019-ICML} values to form the similarity matrix which are then used by the \kdpp\ to sample the subset on neural networks for backpropagation. The motivation for sampling a subset of networks came from a hypothesis which we show in~\Cref{subsec:hypothesis} that the Q-values from the critics converge prematurely during training thus eliminating the need of training all the critics at every training step. 

Additionally, we show that in the event that the CKA matrix is not positive semi-definite, the closest positive semi-definite matrix is just a diagonal perturbation of the CKA matrix and its resulting kernel matrix is still Hermitian positive semi-definite.

We applied \algoname\ on REDQ and performed experiments on MuJoCo environments~\citep{Todorov-2012-IROS} and show that a simple sampling technique can significantly reduce the training time and computation while maintaining similar performance as training all the networks in the ensemble.

To summarize, our contributions are following:
\begin{enumerate}
    \item We empirically show that neural network based value-function approximators collapse prematurely during training in ensemble reinforcement learning.
    \item To address this issue, we propose a Determinantal Point Process based Neural Network Sampler that samples a subset of value-function approximators for backpropagation at every training step. 
    \item We apply \algoname\ on REDQ, that uses an ensemble of \textbf{ten} critic networks. Our experiments have shown that \algoname\ sampling achieves similar or better results than REDQ in $50\%$ computation when measured in FLOPS.
    \item We also provide a theoretical analysis and proof that shows that \kdpp\ sampling of action-value functions leads to lower action-value minimization variance than random sampling $k$ action-values. Additionally, we show how sufficiency conditions for \kdpp\ sampling can easily be met for the Deep RL use case.
\end{enumerate}





\section{Related Work}
\paragraph{Ensembles in Deep Reinforcement Learning:}
Application of ensemble of neural networks in deep reinforcement learning has been studied in several recent studies for different purposes. In~\citep{fujimoto2018addressing,Anschel-2017-ICML,Lan-2020-ICLR,Chen-2021-ICLR} have used an ensemble to address the overestimation bias in deep Q-learning based methods for both continuous and discrete control tasks.  Most recently proposed TOP~\cite{moskovitz-2021-NeurIPS} proposed a method which learns to balance optimistic and pessimistic value estimation online
Similarly, Bootstrapped DQN and extensions~\citep{osband2016deep,chen2017ucb} have used ensemble of neural networks for efficient exploration. In~\cite{Chen-2021-ICLR,mai2022sample} have used large number of ensembles to provide sample efficient reinforcement learning algorithms. Use of ensemble is rapidly growing in offline reinforcement learning to address issue such as error propagation and uncertainty estimations. The error propagation problem in offline reinforcement learning is addressed in~\citep{Kumar-2019-NeurIPS} using ensemble. In recently proposed methods such as~\cite{an2021uncertainty,ghasemipour2022why} have used large number of ensembles to measure uncertainty in offline RL setting. Application of ensemble is not only limited to the critics but several recent papers have used ensembles in the policy domain as well~\cite{lee2020sunrise,zhang2019ace}. 

\paragraph{Determinantal Point Process in Machine Learning:}
 Determinantal Point Processes (DPPs) have emerged as powerful models in the machine learning community in applications requiring information diversity, coverage, and reduced redundancy such as text summarization~\citep{Kulesza-2012-FTML}. Applications of DPPs include video summarization~\citep{Gong-2014-NIPS,sharghi-2016-ECCV}, pose estimation~\citep{Gupta} and  wardrobes creation~\citep{Hsaio-2018-CVPR}. More recently DPPs have been used in reinforcement learning  to promote behavior diversity~\citep{pmlr-v139-lupu21a}. $k$-DPPs~\cite{Kulesza-2011-ICML}, an extension of DPP have been adopted in many applications such as image search and stochastic gradient descent using diversified fixed size mini-batches~\citep{Zhang-2017-AUI}.

\section{Background}
\paragraph{REDQ:} REDQ~\citep{Chen-2021-ICLR} is an off-policy actor-critic method based on max-min RL framework. REDQ uses an ensemble of neural networks to model the critic. One key feature of REDQ is in-target minimization that samples a subset of neural networks to create the target value to train the critics networks. The target value $y$ is calculated as 

\[
y=r+\gamma\left(\min _{i\in \cal{M}} Q_{\phi_{\text {targ},i}}\left(s^{\prime}, \tilde{a}^{\prime}\right)-\alpha \log \pi_{\theta}\left(\tilde{a}^{\prime} \mid s^{\prime}\right)\right),
\] 
where $\cal{M}$ is the number of target networks. The policy gradient is written as 

\[
\nabla_{\theta} \frac{1}{|B|} \sum_{s \in B}\left(\frac{1}{N}\sum_{i =1}^{N} Q_{\phi_{i}}\left(s, \tilde{a}_{\theta}(s)\right)-\alpha \log \pi_{\theta}\left(\tilde{a}_{\theta}(s) | s\right)\right),
\]
where $N$ is the number of critic networks.

\begin{figure*}[h!]
\begin{center}
    \subfloat[Ant-v2]{
\includegraphics[width=\figwidththree]{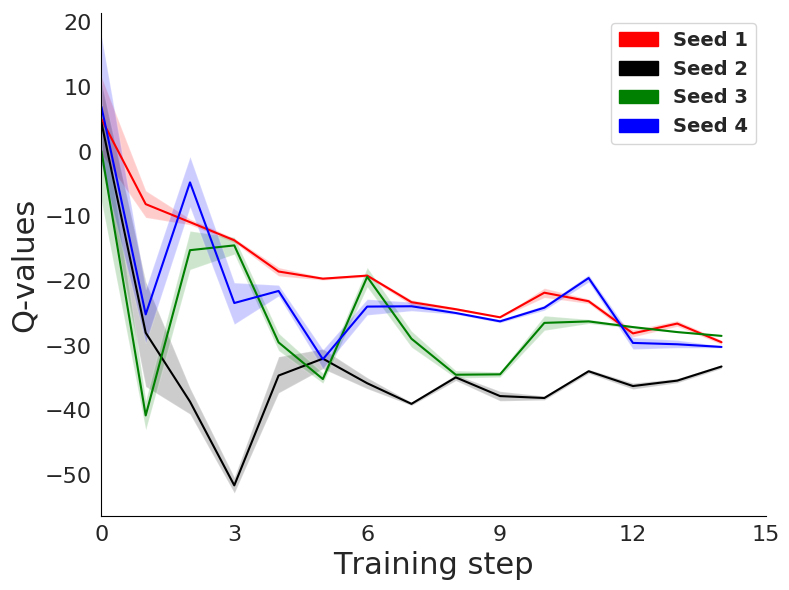}}
    \subfloat[HalfCheetah-v2]{
\includegraphics[width=\figwidththree]{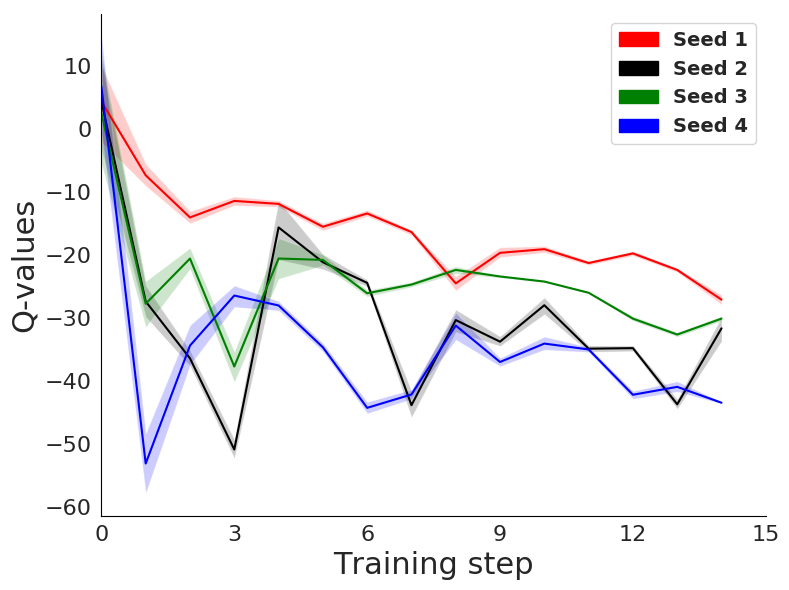}}
    \subfloat[Walker-v2]{
\includegraphics[width=\figwidththree]{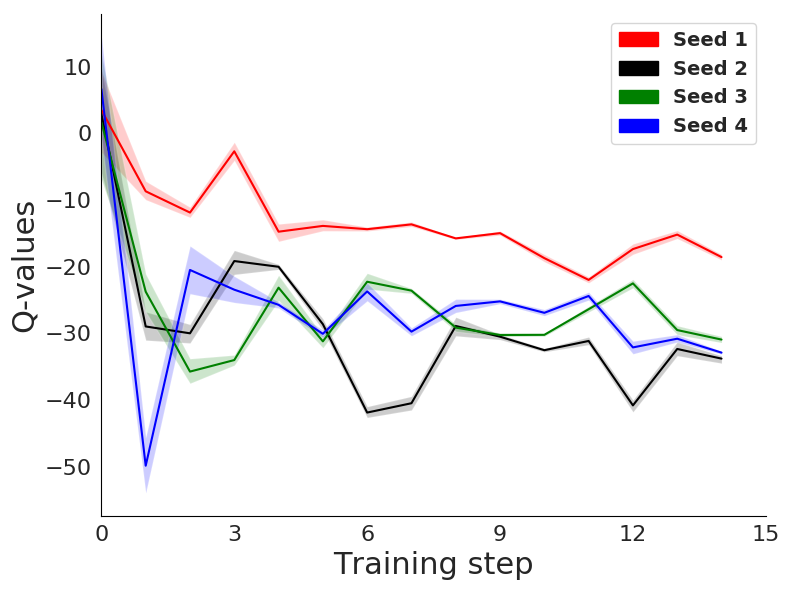}}
\end{center}

\caption{\label{fig:early_q_values} Q-value plots of three MuJoCo environments accumulated over four different seeds. Each curve in the plot represents the mean of the Q-values from ten critics and shaded area represents 95\% confidence interval band. Notice that each curve has a variance in the Q-values in the beginning of the training but it quickly disappears as training continues.}
\end{figure*}

\paragraph{Centered Kernel Alignment:}
Centered Kernel Alignment (CKA)~\citep{Cristianini-2002-NEURIPS,Cortes-2012-JMLR,Kornblith-2019-ICML} is an invertible linear transformation invariant statistic for measuring meaningful multivariate similarity between representations of higher dimension. CKA is a normalized form of Hilbert-Schmidt Independence Criterion (HSIC)~\citep{Gretton-2005-ALT}. Formally, CKA is defined as:
 
 Let $\boldsymbol{X} \in \mathbb{R}^{n\times p_1}$ denote a matrix of activations of $p_1$ neurons for $n$ examples and $\boldsymbol{Y} \in \mathbb{R}^{n\times p_2}$ denote a matrix of activations of $p_2$ neurons for the same $n$ examples. Furthermore, we consider $\boldsymbol{A}_{ij}=a\left(x_i, x_j\right)$ and $\boldsymbol{B}_{ij}=b\left(y_i, y_j\right)$ where $k$ and $l$ are two kernels.

\[
\text{CKA}\left(\boldsymbol{A}, \boldsymbol{B}\right) = \cfrac{\text{HSIC}\left(\boldsymbol{A}, \boldsymbol{B}\right)}{\sqrt{\text{HSIC}\left(\boldsymbol{A}, \boldsymbol{A}\right)\cdot \text{HSIC}\left(\boldsymbol{B}, \boldsymbol{B}\right)}} 
\]
HSIC is a test statistic for determining whether two sets of variables are independent. The empirical estimator of HSIC is defined as:
\[
\text{HSIC}\left(\boldsymbol{A}, \boldsymbol{B}\right) = \cfrac{1}{(n-1)^2} \Tr\left(\boldsymbol{AHBH}\right)
\]
where $H$ is the centering matrix $H_n = I_n - \cfrac{1}{n}\mathbf{11}^T$.

\textbf{Determinantal Point Processes}:
A Determinantal point process (DPP)~\citep{Macchi-1975-AAP} is a random point process useful for the combinatorial problem of selecting a diverse sample from a set.  In particular, a DPP for a given finite set defines a probability distribution over subsets, where subsets containing diverse items have high probability and are thus more likely to be selected. We briefly discuss finite determinantal point processes here, for in-depth discussions refer~\citep{Hough-2006-PS,Kulesza-2012-FTML,Li-2016-ICML,Derezinski-2019-NeurIPS}.


\begin{definition}
A point process $\boldsymbol{X}$ on discrete set $\mathcal{S}$ and with Hermitian positive semi-definite marginal kernel $\boldsymbol{K}$:$\mathcal{S}^2\rightarrow\mathbb{C}$, $\boldsymbol{K}\preceq 1$ (all eigenvalues of $\boldsymbol{K}$ are at most 1) is called \textit{determinantal} iff 
\begin{equation}
P(\boldsymbol{X}\supseteq(x_i,\ldots,x_n))=det(\boldsymbol{K}(x_i,x_j))_{1\leq i,j\leq n}
\end{equation}
for any $n\in\mathbb{Z}^{+}$ and any $x_i,\ldots,x_n$ $\in\mathcal{S}$ or equivalently, iff
\begin{equation}
\label{eqn:K}
P(\boldsymbol{X}\supseteq x)=det(\boldsymbol{K}_x)
\end{equation}
for any $x\subset\mathcal{S}$, where $\boldsymbol{K}_x$ is the submatrix of $\boldsymbol{K}$ indexed by $x\times x$.
\end{definition}
Consequently, DPPs are a repulsive distribution over set $\mathcal{S}$, generating subsets that exhibit diversity.

Furthermore, for the case when $\boldsymbol{I}-\boldsymbol{K}$ invertible, the DPP $\boldsymbol{X}$ is called an $L$-ensemble with kernel $\boldsymbol{K}:=\boldsymbol{I}-(I+\boldsymbol{L})^{-1}$ and distribution
\begin{equation}
\label{eqn:l-ensemble}
P(\boldsymbol{X}=x)=det(\boldsymbol{L}_x)det(\boldsymbol{I}+\boldsymbol{L})^{-1}
\end{equation}
for any $x\subset\mathcal{S}$, where $\boldsymbol{L}_x$ is the sub-matrix of $\boldsymbol{L}$ indexed by $x\times x$.

\renewcommand\thetheorem{1}
\begin{lemma}\citep{Collings-1983-TAS}
Let $D$ be an $N\times N$ diagonal matrix and let $M$ be an arbitrary $N\times N$ matrix. The determinant of $(D + M)$ is :
\begin{equation}
det(\boldsymbol{D}+\boldsymbol{M}) = \sum_{S\subseteq\mathcal{S}}det(\boldsymbol{D}_S)det(\boldsymbol{M}_S).
\end{equation}
\end{lemma}
Thus (\ref{eqn:l-ensemble}) can be re-written in normalized form as:
\begin{equation}
\label{eqn:l-ensemble2}
P(\boldsymbol{X}=x)=det(\boldsymbol{L}_x)(\sum_{x\subseteq\mathcal{S}}det(\boldsymbol{L}_x))^{-1}.
\end{equation}
In this paper, we only considers DPPs that are $L$-ensembles because of their advantages such as:
\begin{enumerate}[i.]
\item While (\ref{eqn:K}) gives the probability that a set is contained in the DPP, (\ref{eqn:l-ensemble}) gives the exact probability that a sampled set is from the DPP and is thus more relevant for set selection tasks requiring samples from different regions in a feature space. From (\ref{eqn:l-ensemble}), more diverse sets have higher probability and are thus more likely to be selected. 
\item There is no requirement that all eigenvalues of $\boldsymbol{L}$ are less than or equal to 1.
\end{enumerate}
Since standard DPP sampling does not provide the flexibility of sampling a pre-specified size, in this work we focus on $k$-Determinantal Point Processes ($k$-DPPs). A $k$-DPP on discrete set $\mathcal{S}$ is a distribution over all subsets of $\mathcal{S}$ with cardinality $k$ and is thus a conditioning of a standard DPP on the event that a subset $X$ of $\mathcal{S}$ has a fixed size. A $k$-DPP thus gives probabilities
\begin{equation}
P^{\boldsymbol{(k)}}(\boldsymbol{X}=x)=det(L_x)(\sum_{\substack{x
\subseteq\mathcal{S}\\|x|=k}}det(\boldsymbol{L}_x))^{-1},
\end{equation}
where $|x|=k$ and $\boldsymbol{L}$ is a positive semi-definite kernel.\citep{Kulesza-2011-ICML}
Because $k$-DPPs only model contents of a set, they are less costly than standard DPPs and are useful in situations where sample size is constrained,
for example by empirical bounds or hardware restrictions \citep{Zhang-2017-AUI}.
\begin{figure*}
\begin{center}
    \subfloat[Ant-v2]{
\includegraphics[width=\figwidththree]{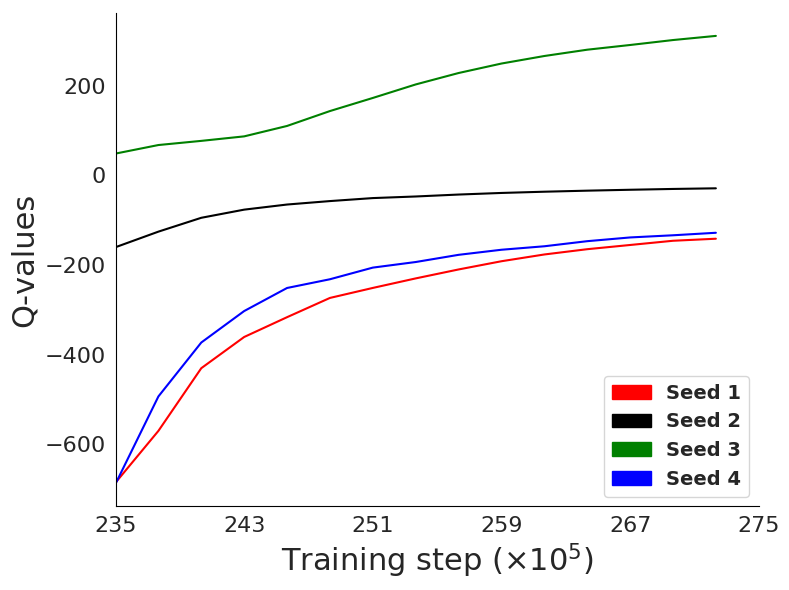}}
    \subfloat[HalfCheetah-v2]{
\includegraphics[width=\figwidththree]{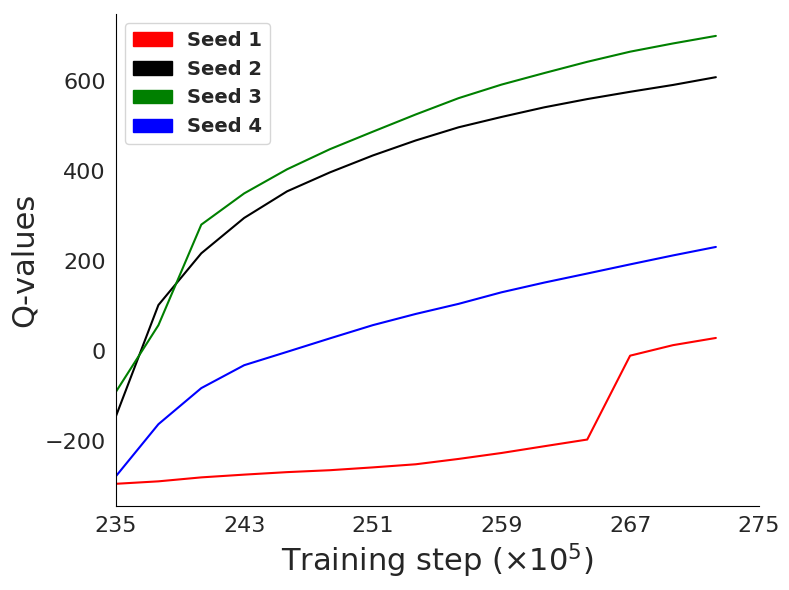}}
    \subfloat[Walker-v2]{
\includegraphics[width=\figwidththree]{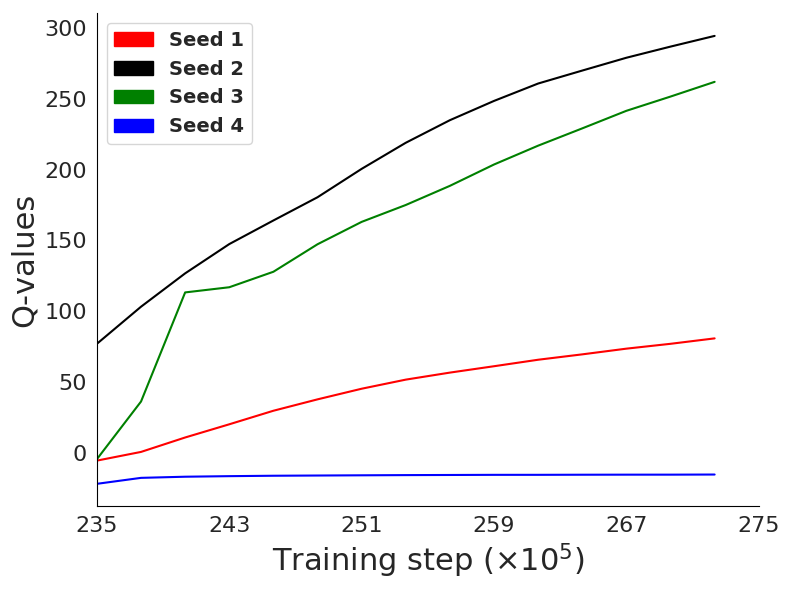}}
\end{center}

\caption{\label{fig:q_values_end} Q-value plots of three MuJoCo environments accumulated over four different seeds at the tail end of the training. Notice that the Q-values from the critics have completely collapsed and has zero variance.}
\end{figure*}

\section{\algoname: Determinantal Point Process Based Neural Network Sampler}
In this section, we propose \algoname: \textbf{D}eterminantal Point Process Based \textbf{N}eural Network \textbf{S}ampler that samples a subset of critic networks for backpropagation at training time. In principal \algoname\ can be used with any off-policy algorithm that uses the same target value to train the critics such as REDQ~\cite{Chen-2021-ICLR}, TOP~\citep{moskovitz-2021-NeurIPS} MaxminDQN~\cite{Lan-2020-ICLR}, EnsembleDQN~\cite{Anschel-2017-ICML}. For the exposition, we describe only the REDQ version in this paper.

This section is organized as follows: 
\begin{enumerate}
    \item We {\em empirically} show that the Q-values from the critic networks collapse prematurely during training time.
    \item we present the \kdpp\ based sampling algorithm to sample the indices for critic networks to train.  
\end{enumerate}

\subsection{ Empirical Evidence of Early Collapse of Q-values}
\label{subsec:hypothesis}
The work on this paper starts with a conjecture that the Q-values from the critic networks collapse prematurely. To verify our hypothesis, we trained REDQ on HalfCheetah-v2, Ant-v2 and Walker2d-v2 on four different seeds and measured the Q-values from all the ten critics. As shown in~\Cref{fig:early_q_values},  it took around fifteen training steps for all the ten critics from having distinct  Q-values to collapse to almost identical values in nearly every run for all three environments. Note that each curve in the plot represents the mean value of all ten critics and the shaded area around the curve represents 95\% confidence interval. 

A counter argument can be made here that in~\Cref{fig:early_q_values} we did not allow enough training steps that might induce any variance in the Q-values. To address this, we measured the Q-values at the tail end of the training. As shown in~\Cref{fig:q_values_end}, the Q-values have completely collapsed in all the runs across all three environments.  From this evidence, we can conclude that longer training reduces the variance in the Q-values.

\subsection{Compute Efficient Neural Network Sampling}
The motivation behind the idea of training a subset of critic networks came from the observation in~\Cref{subsec:hypothesis} that if all the critic networks collapse early in the training, we can only train a subset of critic networks at every training step. This will allow us to force diversity in the Q-values which recently has been shown to be a key component in ensemble reinforcement learning~\citep{sheikh2022maximizing}. To sample a diverse set of critics, we use \kdpp~\citep{Kulesza-2011-ICML} which is a derivative of DPPs~\cite{Macchi-1975-AAP}. The advantage of \kdpp\ over DPP is that \kdpp\ allows us to have a control on the size of the sampled neural networks while standard DPP automatically selects the size of the subset. One key component required for using the \kdpp\ is a similarity matrix. Since we are interested in sampling critic networks with diverse Q-values, we created the similarity matrix by measuring the pairwise CKA similarity of all the Q-values. Formally, the similarity matrix $\boldsymbol{L}\in\mathbb{R}^{^{N\times N}}$ is defined as:
\begin{equation}
\label{eq:sim_matrix}
\boldsymbol{L}=\text{CKA}((Q_{\phi_i}(s, a), Q_{\phi_j}(s, a))_{1 \leq i,j\leq N}.    
\end{equation}
The similarity matrix $\boldsymbol{L}$ is used by the \kdpp\ to sample the indices of the diverse critics to train.

\begin{figure*}[t]
\begin{center}
    \subfloat[Ant-v2]{
\includegraphics[width=\figwidthfour]{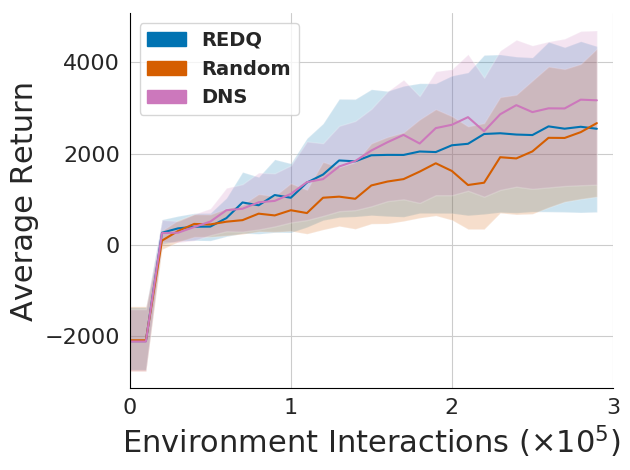}}\hspace{-6pt}
     \subfloat[HalfCheetah-v2]{
\includegraphics[width=\figwidthfour]{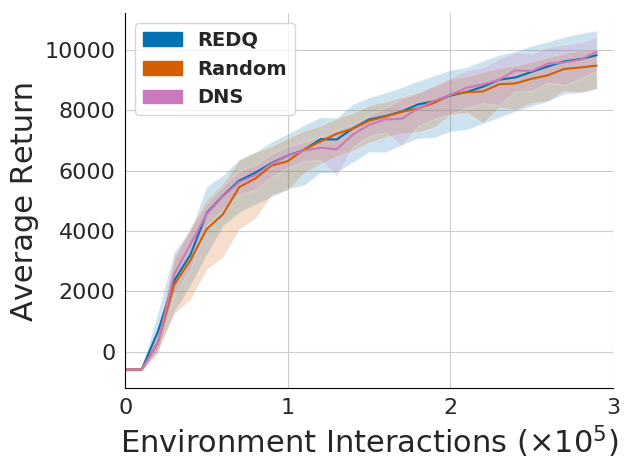}}\hspace{-10pt}
     \subfloat[Hopper-v2]{
\includegraphics[width=\figwidthfour]{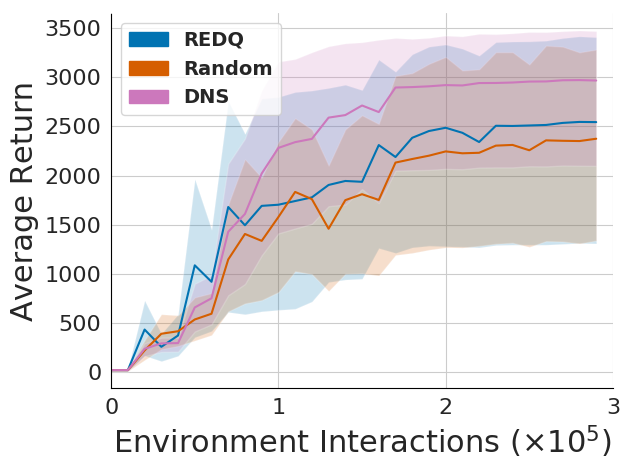}}\hspace{-10pt}
\subfloat[Walker2d-v2]{
\includegraphics[width=\figwidthfour]{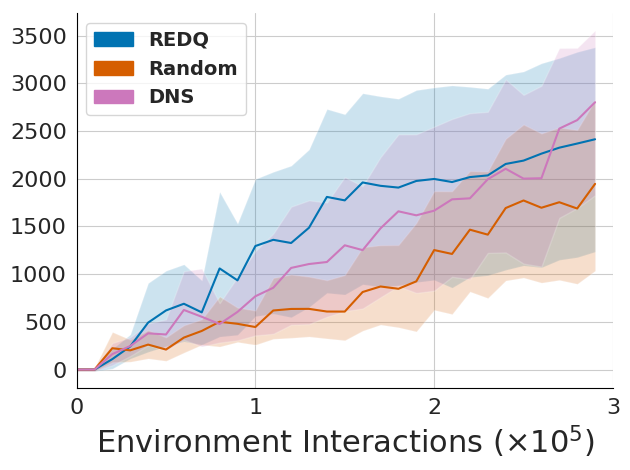}}\hspace{-10pt}
\end{center}
\caption{\label{fig:main_results}Training curves with 95\% confidence interval of baseline REDQ, random sampling and DNS.}
\end{figure*}

\begin{table*}[h]
\small
\centering
\caption{Max average return for 10 runs of 300K time steps. Maximum value for each task is bolded. $\pm$ corresponds to a single
standard deviation over runs}
\label{table:main_results}
\vskip 0.15in
\setlength{\tabcolsep}{4pt}
\renewcommand{\arraystretch}{1.5}
{
\begin{tabular}{lccccc}
\toprule
\bf{Environment} & \bf{Baseline} & \bf{Random} & \bf{DNS}  \\
\midrule
Ant-v2 	& 2543.1 $\pm$ 2595.7  & 2666.8 $\pm$ 2262.6  & \textbf{3167.2 $\pm$ 2484.7} \\
HalfCheetah-v2 & 9818.8 $\pm$ 1445.2 & 9474.3 $\pm$ 991.1 & \textbf{9931.0 $\pm$ 819.1} \\
Hopper-v2 & 2544.2 $\pm$ 1468.21 & 2374.9 $\pm$ 1405.8 & \textbf{2967.8 $\pm$ 1128.9} \\
Walker2d-v2 & 2414.4 $\pm$ 1580.0 & 1946.4 $\pm$ 1287.9 & \textbf{2802.3 $\pm$ 1272.1}\\
\bottomrule
\end{tabular}
}
\vskip -0.1in
\end{table*}
Formally, we consider a REDQ agent with $N$ number of critic networks parameterized by $\{\phi_{i}\}_{i=1}^{N}$. At every training step, we sample a batch of experience $B$ from replay buffer $\mathcal{D}$ . Using state-action $(s,a) \in B$, we fetch the Q-values $Q_{\phi_{i}}(s, a) \text{ for } i =1, 2, \ldots, N$. Using the Q-values, we create the similarity matrix $\boldsymbol{L}$ by measuring the pairwise CKA similarity using~\Cref{eq:sim_matrix}. The similarity matrix $\boldsymbol{L}$ is then used by \kdpp\ to sample a diverse set of critic networks of size $k$ to train. The rest of the training process is identical to the baseline REDQ which we invite the readers to see in~\Cref{alg:DNS}. {\em Note that the output of the \kdpp\ is indices of the critic networks}.

\subsection{Formal Theoretical Analysis}

In \algoname\,we utilize CKA values as entries of the similarity matrix  $\boldsymbol{L}\in\mathbb{R}^{^{N\times N}}$:
\[
\boldsymbol{L}=\text{CKA}((Q_{\phi_i}(s, a), Q_{\phi_j}(s, a))_{1 \leq i,j\leq N}.
 \]
Since not all similarity matrices are positive semi-definite, $\boldsymbol{L}$ can be approximated with the closest positive semi-definite matrix such that the relative similarity strengths among point pairs are preserved.
\begin{proposition}
The nearest positive semi-definite matrix to a symmetric matrix to
$\boldsymbol{L}\in\mathbb{R}^{^{N\times N}}$ is a diagonal perturbation of $\boldsymbol{L}$:
\[
\overset{\sim}{\boldsymbol{L}}=\boldsymbol{L}+\boldsymbol{D},
\]
where $\boldsymbol{D}=diag((\lambda_i+|\lambda_i|)/2)$, and
$\lambda_i, i\in\{1,\ldots,N\}$ are eigenvalues from the spectral
decomposition of $\boldsymbol{L}$.

\end{proposition}

\begin{proof}
The nearest positive semi-definite matrix $\overset{\sim}{\boldsymbol{L}}$ to a matrix $\boldsymbol{L}$ can be computed via a spectral decomposition of $\boldsymbol{B}=(\boldsymbol{L}+\boldsymbol{L}^T)/2=\boldsymbol{V}\Lambda\boldsymbol{V}^T=\boldsymbol{V}diag(\lambda_i)\boldsymbol{V}^T$ as:
\[
\overset{\sim}{\boldsymbol{L}}=\boldsymbol{V}diag(d_i)\boldsymbol{V}^T,\]
where
\[d_i=
\begin{cases}
\lambda_i,&\lambda_i\geq 0\\
0,&\lambda_i<0
\end{cases}
\]
or equivalently as: 
\begin{equation}
\label{eqn:psd}
\overset{\sim}{\boldsymbol{L}}=(\boldsymbol{B}+\boldsymbol{H})/2,
\end{equation}
where $\boldsymbol{H}=\boldsymbol{V}diag(|\lambda_i|)\boldsymbol{V}^T$\citep{higham-1988-LAA}.
Since $\boldsymbol{L}$ is symmetric, $\boldsymbol{B}=\boldsymbol{L}=\boldsymbol{V}diag(\lambda_i)\boldsymbol{V}^T$ and
$\boldsymbol{H}=\boldsymbol{V}diag(|\lambda_i|)\boldsymbol{V}^T.$ It follows from (\ref{eqn:psd}) that $\overset{\sim}{\boldsymbol{L}}=\boldsymbol{L}+\boldsymbol{D}$ where $\boldsymbol{D}=(diag(\lambda_i)+diag(|\lambda_i|))/2=diag((\lambda_i+|\lambda_i|)/2)$.
\end{proof}

\begin{figure*}[th!]
\begin{center}
     \subfloat[\label{fig:flops}Computation cost of the backpropagation method in terms of petaFLOPS]{
\includegraphics[width=\figwidthtwo]{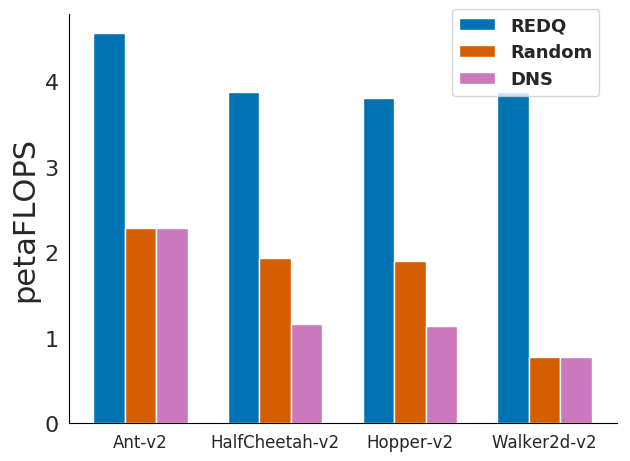}}\hspace{-10pt}
\subfloat[\label{fig:clock}Average wall-clock training time in hours]{
\includegraphics[width=\figwidthtwo]{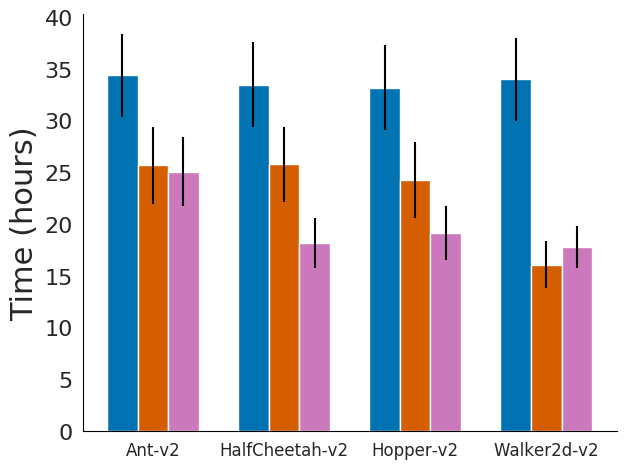}}\hspace{-6pt}
\end{center}
\caption{\label{fig:compute}Both bar graphs represent computation cost of the experiments shown in~\Cref{fig:main_results} in terms of FLOPS and wall-clock time.}
\end{figure*}

Next, recall that under the tabular version of REDQ a subset of action-value functions are updated to $Q_{t+1}^l(s,a)\,=\,Q_t^l(s,a)+\alpha[Y^{MQ}-Q_t^l(s,a)]$ and at $t+1$ an action $a$ is sampled according to the minimum of a random sample $M$ of $N$ Q$\,$functions, such that $|\cal{M}|=M$ ,$\,$i.e. according to $\underset{i\in \cal{M}}{\text{min}}\,Q_{t+1}^i(s,a)$. In REDQ, $M=2$.

For $i\in \cal{M}$, \ Let $I_i\in\{0,1\}$ be a random variable indicating if $Q^i(s,a)$ was updated, i.e. $I_i\sim$Bernoulli$(p_i)$. Then 
\begin{multline}
Q_{t+1}^i(s,a)\,=\,Q_t^i(s,a)+\alpha I_i[Y^{MQ}-Q_t^i(s,a)]\\
=Q_t^\pi(s,a)
+\varepsilon^i_t(s,a) 
+\alpha
I_i[Y^{MQ}
-Q_t^\pi(s,a)-\varepsilon^i_t(s,a)]
\end{multline}
Furthermore, assume that approximation errors $\varepsilon^i_t(s,a)$ are
identically distributed $U(-\tau$,$\tau)$ for each fixed $(s,a)$.
The$\,$theorem below characterizes the relationship between $k$-DPP sampling and the variance$\,$of $\underset{i\in \cal{M}}{\text{min}}\,Q_{t+1}^i(s,a)$ and $\frac{1}{M}\sum_{i\in \cal{M}}Q_{t+1}^i(s,a)$. 
\begin{theorem}
\label{theorem:Theorem 1}
Under the conditions above and for set $\cal{M}$ of M random samples of $N$ Q functions,
\[Var\,Q^{min}=Var(\underset{i\in \cal{M}}{\text{min}}\,Q_{t+1}^i(s,a)|Y^{MQ})\,
\]
decreases if for some $i,j\in \cal{M}$ $Q_t^i(s,a)$ and $Q_t^j(s,a)$ were sampled pre-update according to $k$-DPP. Variance reduction can also be shown for the sample mean 
\[Var\,Q^{avg}=Var(\frac{1}{M}\sum_{i\in \cal{M}}Q_{t+1}^i(s,a)|Y^{MQ}).\] 
Additionally, $Var\,Q^{min}$ and
$Var\,Q^{avg}$ are lower under $k$-DPP than under $k$-random sampling.
\end{theorem}
  \Cref{theorem:Theorem 1} shows why $k$-DPP sampling boosts performance over $k$-random sampling. For the special case that$\,$all $N\,$Q functions are close to being dissimilar,  $k$-DPP sampling $k$ of the $N$ Q-functions approaches uniform $k$-random sampling with $P_{\cal{K}}=\frac{1}{{N \choose k}}$ for all sets of $\cal{K}$ size $k$. We summarize this in the corollary below and note that under the $k$-DPP scheme, just as in $k$-random sampling some variance is retained, which is beneficial for exploration. 

\begin{corollary}
If all $N$ Q-functions are completely dissimilar, $k$-DPP sampling is
equivalent to $k$-random uniform sampling with each set $\cal{K}$ with cardinality $k$ having equal probability $P_{\cal{K}}=\frac{1}{{N \choose k}}$.
\end{corollary}
We show the proofs of \Cref{theorem:Theorem 1} in the appendix. Corollary 1 follows from the fact that when the network activations are completely dissimilar, the off-diagonal elements of the $L$-matrix are 0 since CKA values are zero. Thus in this case the $L$-matrix is just the identity matrix and the resulting sampled item probabilities are equal by equation (\ref{eqn:l-ensemble2}).  

\begin{figure*}[th!]
\begin{center}
    \subfloat[Ant with two and three\\neural networks\label{fig:ant-2}]{
\includegraphics[width=\figwidthfour]{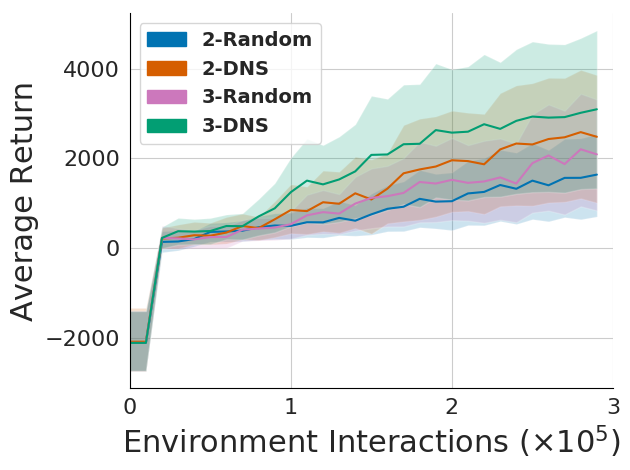}}\hspace{-6pt}
     \subfloat[Ant with four and five\\neural networks\label{fig:ant-4}]{
\includegraphics[width=\figwidthfour]{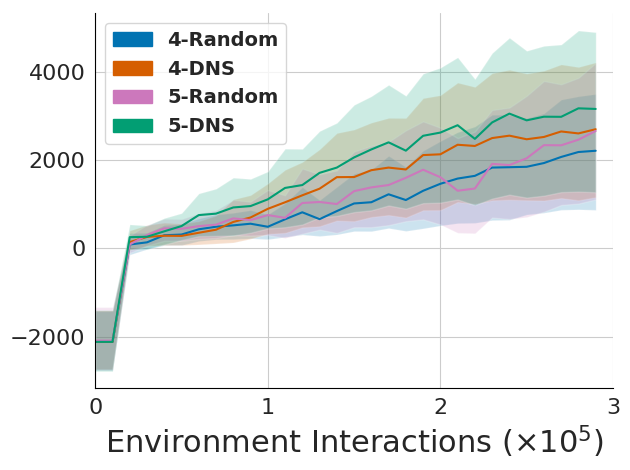}}\hspace{-10pt}
     \subfloat[Walker2d with two and\\three neural networks\label{fig:walker-2}]{
\includegraphics[width=\figwidthfour]{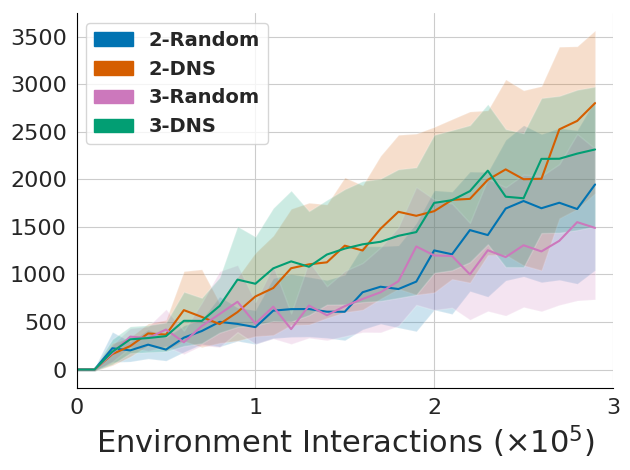}}\hspace{-10pt}
\subfloat[Walker2d with four and five\\neural networks\label{fig:walker-4}]{
\includegraphics[width=\figwidthfour]{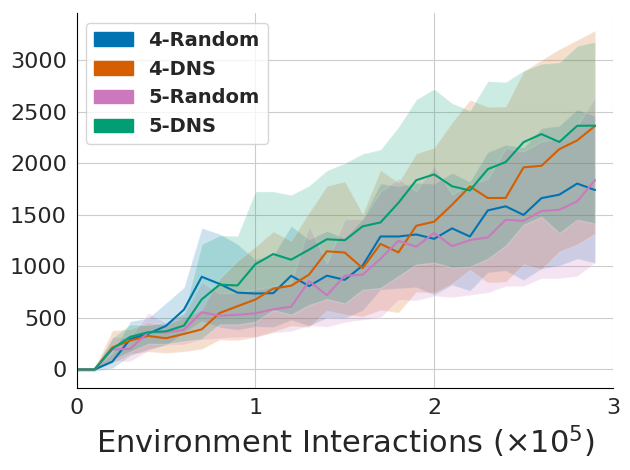}}\hspace{-10pt}
\end{center}
\caption{Training curves for Ant-v2 and Walker-v2 environments for varying values of $k$ for both DNS and random sampling.\label{fig:random}}
\end{figure*}

\begin{figure*}[t!]
\begin{center}
    \subfloat[Ant-v2]{
\includegraphics[width=\figwidthfour]{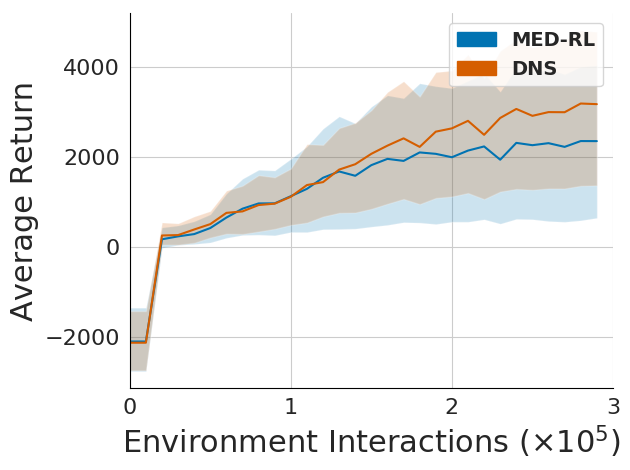}}\hspace{-6pt}
     \subfloat[HalfCheetah-v2]{
\includegraphics[width=\figwidthfour]{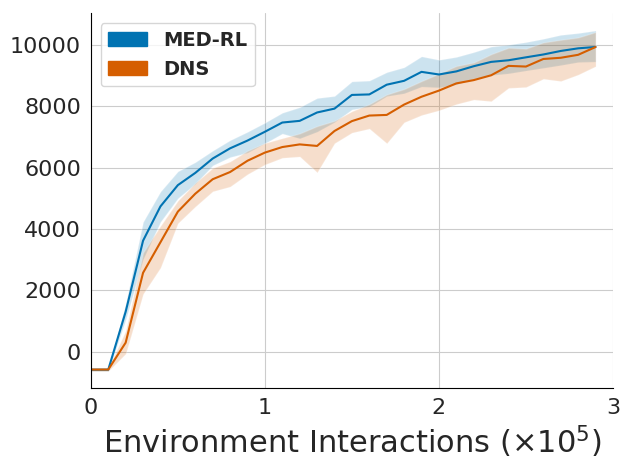}}\hspace{-10pt}
     \subfloat[Hopper-v2]{
\includegraphics[width=\figwidthfour]{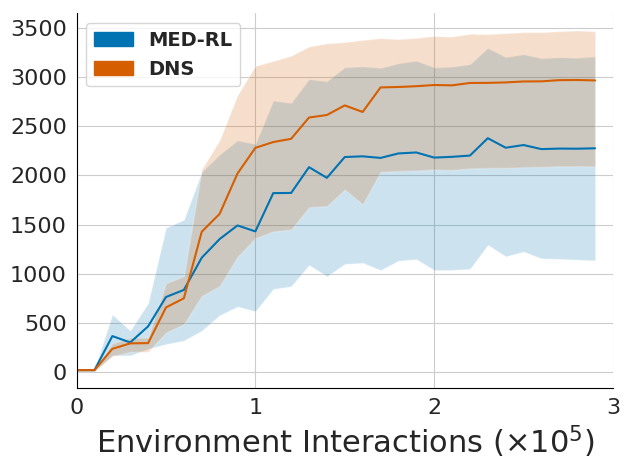}}\hspace{-10pt}
\subfloat[Walker2d-v2]{
\includegraphics[width=\figwidthfour]{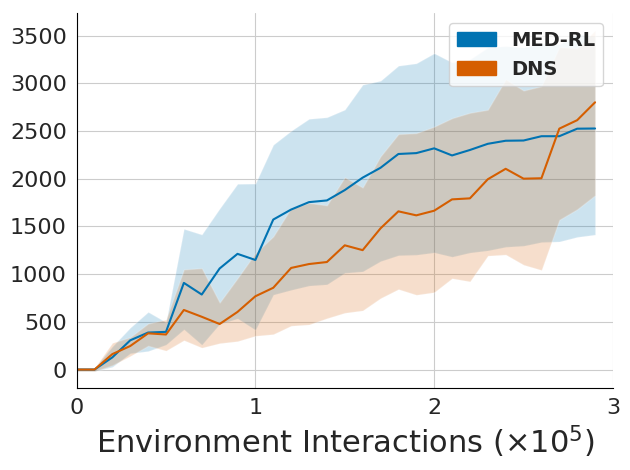}}\hspace{-10pt}
\end{center}
\caption{\label{fig:gini} Training plots comparing DNS with MED-RL, a regularization method that uses Gini coefficient to maximize diversity.}
\end{figure*}

\section{Experiments}\label{sec:experiments}
We designed our experiments to answer the following questions:
\begin{enumerate}
    \item Can \algoname\ match the performance baseline REDQ while training only a subset of critic networks?
    \item Is \algoname\ better than random sampling?
    \item Is DNS better than diversity regularization?
    \item Does choice of $k$ matter for \algoname?.
\end{enumerate} 
\subsection{Experimental Setup}
We evaluate DNS on several different continuous control tasks from MuJoCo~\citep{Todorov-2012-IROS} and compare DNS with baseline REDQ where all the ten critics are trained at every training step and random sampling of neural networks for training. For DNS and random sampling, we sampled between and two and five networks for our experiments. Following the setup, we report the highest returns after 300K environment interactions on Ant-v2, HalfCheetah-v2, Hopper-v2 and Walker2d-v2 environments. We report the mean and the standard deviation across ten runs in~\Cref{fig:main_results}. For clarity, the results are also shown in~\Cref{table:main_results}.
From~\Cref{fig:main_results,table:main_results}, we can see that DNS consistently outperforms REDQ and random sampling on Ant-v2, Hopper-v2 and Walker2d-v2 and matches the performance of REDQ on HalfCheetah-v2 environment. Note that the goal of this paper is to match the performance of REDQ while reducing the computation cost. For that reason we did not perform any hyperparameter tuning. All the hyperparameters such as learning rate, batch size, neural network size and the seeds were kept fixed across all the experiments. The only hyperparameter that has been tuned in this paper is $k$ which samples the number of neural networks.

Details of the hyperparameters used in our experiments are shown in~\Cref{sec:params}

\subsection{Computational Analysis}
We measured the computation cost of the experiments shown in~\Cref{fig:main_results} to verify that sampling a subset of critic networks at training indeed reduces the computation cost. We measured the computation cost in terms wall-clock time which is a subjective metric and depends on the computing infrastructure and in FLOPS which regard as hardware independent metric. Since we are interested in reducing the backpropagation steps, we wrapped the {\em Backward()} function in Pytorch's profiler and measured the FLOPS needed to compute the {\em Backward()} function. We then multiplied the obtained FLOPS with total number of training steps. The resulting plot is shown in~\Cref{fig:flops}.

Similarly for measuring the wall-clock time, we wrapped the whole training procedure by CodeCarbon~\citep{codecarbon}. Since wall-clock time is subjective and can be affected by multiple factors, we calculated the average with standard deviation. The resulting plot is shown in~\Cref{fig:clock}. From~\Cref{fig:compute}, we can see that DNS achieved better performance than baseline REDQ in at least 50\% FLOPS. The key point to note that DNS achieved 15\% more average cumulative reward in less than 25\% of FLOPS as compared to baseline REDQ on Walker2d-v2 environment.

\begin{figure*}[t!]
\begin{center}
    \subfloat[Ant-v2]{
\includegraphics[width=\figwidthfour]{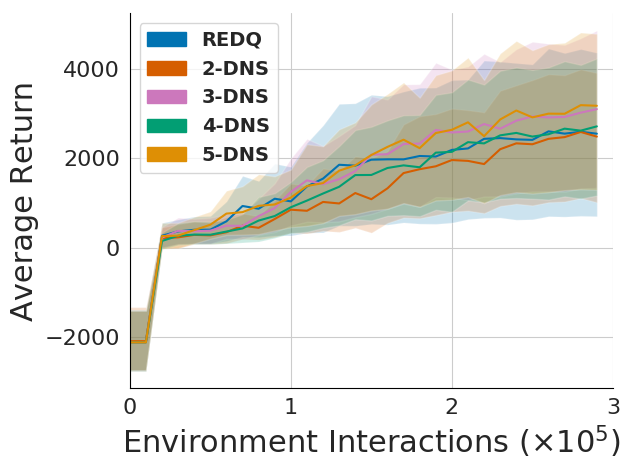}}\hspace{-6pt}
     \subfloat[HalfCheetah-v2]{
\includegraphics[width=\figwidthfour]{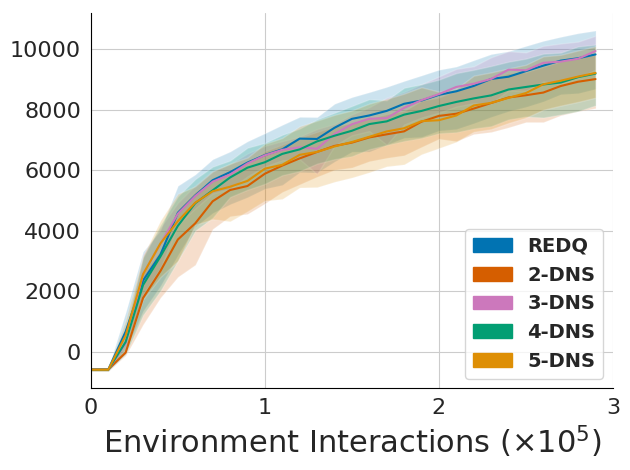}}\hspace{-10pt}
     \subfloat[Hopper-v2]{
\includegraphics[width=\figwidthfour]{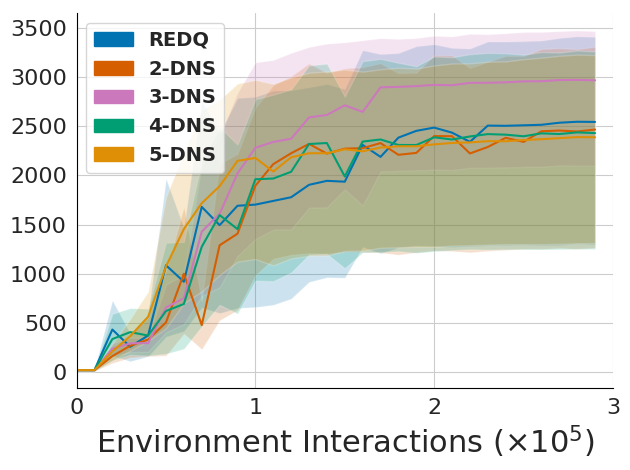}}\hspace{-10pt}
\subfloat[Walker2d-v2]{
\includegraphics[width=\figwidthfour]{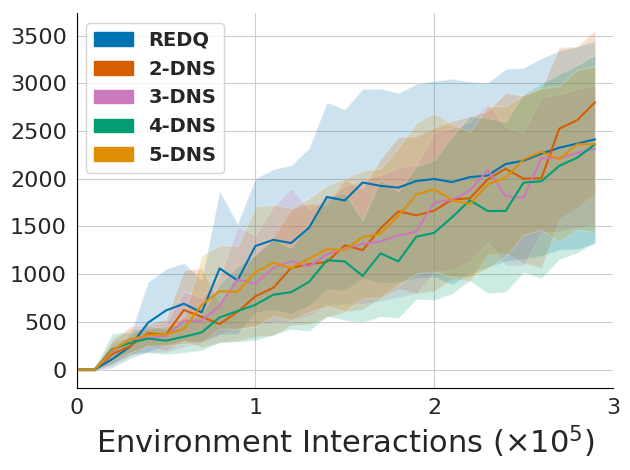}}\hspace{-10pt}
\end{center}
\caption{Training curves showing the effect of $k$ on DNS\label{fig:ablation}}
\end{figure*}

\subsection{Is DNS Better than Random Sampling?}
To address the question whether Random sampling is better than DNS, we plotted the training curves for Ant-v2 and Walker2d-v2 environments for varying values of $k$ ranging between 2 and 5 in~\Cref{fig:random}. To avoid confusion, we split the plots in two figures for each environment. ~\Cref{fig:ant-2,fig:walker-2} shows the training plots for $k=2$ and $k=3$  and ~\Cref{fig:ant-4,fig:walker-4} shows the training plots for $k=4$ and $k=5$. From~\Cref{fig:random}, we can see that for every value of $k$, DNS outperforms random sampling in both environments. 

\subsection{Is DNS Better than Diversity Regularization?}
Recently~\cite{sheikh2022maximizing} proposed MED-RL that uses economics inspired regularizers such as Gini and Theil coefficients to maximize diversity between the neural networks. Since \kdpp\ is an alternate way on inducing diversity, we compared DNS with MED-RL (Gini) where we augmented REDQ with the Gini as proposed in~\citep{sheikh2022maximizing}.~\Cref{fig:gini} shows the training curve for the MED-RL and DNS where we can see that DNS clearly outperforms MED-RL on Ant-v2 and Hopper-v2 environments and matches the performance of MED-RL on HalfCheetah-v2 environment. One quick point to note that our results of MED-RL do not match with the results proposed in~\citep{sheikh2022maximizing}. We attribute this discrepancy to multiple factors such as different learning rates and they have have shown results for five seeds only where we have shown results for ten seeds.

\subsection{Does  choice of $k$ matter?}
To analyze the effect of choice of $k$, we performed an ablation study in which we trained DNS on varying values of $k$ for all four MuJoCo environments and plotted the training curves in~\Cref{fig:ablation}. For the Ant-v2 environment, $k=2$ and $k=4$ performed similar to baseline REDQ. For HalfCheetah-v2 and Walker2d environments, most values of $k$ under-performed when compared with REDQ while for Hopper-v2, $k=3$ outperformed REDQ significantly. 


\begin{table}[t]
\centering
\caption{Hyperparameters for continuous control tasks\label{sec:params}}
\begin{tabular}{|c||c|} 
 \hline
 Hyperparameter & Value\\ 
 \hline \hline
  Target weight $\tau$  & $1e^{-3}$ \\ 
  Actor learning rate   &  $3e^{-4}$ \\ 
  Critic learning rate   & $ 3e^{-4}$\\ 
  Replay buffer   &  $1e^{6}$\\ 
  Batch size    & $256$ \\
  Exploration steps & 25000 \\
  Optimizer & Adam \\ 
  Hidden Llayer size & 256 \\
  Number of critics (REDQ)   & $10$ \\
  Regularization weight & $1e^{-8} $ \\
 \hline \hline
\end{tabular}
\end{table}

\section{Implementation Details and Hyperparameters}
\label{sec:hyperparameters}
For REDQ, we used the code provided by the authors {\textcolor{blue}{https://github.com/watchernyu/REDQ}}. For \kdpp\ we used the DPPy package {\textcolor{blue}{https://github.com/guilgautier/DPPy}}. The complete list of hyperparmeters is given in~\Cref{sec:params}.

\subsection{Computing Infrastructure}
All the experiments were performed on a Kubernetes managed cluster with Nvidia V100 GPUs and Intel Skylake CPUs. Each experiment was run as an individual Kubernetes job with 5 CPUs, 16GB of RAM and 1 GPU.

\section{Conclusion}
In this paper, we proposed \algoname: a \textbf{D}eterminantal Point Process based \textbf{N}eural Network \textbf{S}ampler that specifically uses \kdpp\ to sample a subset of neural networks for backpropagation at every training step. This sampling allowed us to reduced the computation cost by 50\% during training. We evaluated DNS on MuJoCo environments and compared our results with baseline REDQ and random sampling. Additionally, DNS outperformed MED-RL, a regularization method that maximizes diversity between the ensemble of neural networks in deep reinforcement learning.

\clearpage
\bibliography{ref}
\bibliographystyle{icml2021}
\clearpage
\appendix
\onecolumn
\section*{Supplementary}
\section{Algorithm}
\begin{algorithm}[h!]
   \caption{DNS: REDQ version}
   \label{alg:DNS}
\begin{algorithmic}
   \STATE Initialize policy parameters $\theta$,  $N$ Q-function parameters $\phi_{i}$, $i=1,\ldots,N$, empty replay buffer $\mathcal{D}$. Set target parameters $\phi_{\text {targ}, i} \leftarrow \phi_{i}, \text{ for } i =1, 2, \ldots, N$
  \REPEAT
  \STATE Take one action $a_t \sim \pi_\theta(\cdot | s_t)$. Observe reward $r_t$, new state $s_{t+1}$.
  \STATE Add data to buffer: $\mathcal{D} \leftarrow \mathcal{D} \cup \{(s_t, a_t, r_t, s_{t+1})\}$ 
  \FOR {$G$ updates}  
\STATE Sample a mini-batch $B=\left\{\left(s, a, r, s^{\prime} \right)\right\}$ from $\mathcal{D}$
\STATE Fetch $Q_{\phi_{i}}(s, a) \text{ for } i =1, 2, \ldots, N$
\STATE Compute similarity matrix $\boldsymbol{L}$: \[
\boldsymbol{L} = CKA_{i,j\in N} (Q_{\phi_{i}}(s, a), Q_{\phi_{j}}(s, a))
\]

\STATE Sample a set $K$ of  $k$ distinct indices from $\{1, 2, \ldots, N\}$:
\[
    K = DPP(\boldsymbol{L}, k)
\]
\STATE Sample a set $\cal{M}$ of  $M$ distinct indices from $\{1, 2, \ldots, N\}$
\STATE Compute the Q target $y$ (same for all of the $k$ Q-functions):
\[
y=r+\gamma\left(\min _{i\in \cal{M}} Q_{\phi_{\text {targ},i}}\left(s^{\prime}, \tilde{a}^{\prime}\right)-\alpha \log \pi_{\theta}\left(\tilde{a}^{\prime} \mid s^{\prime}\right)\right),\quad \tilde{a}^{\prime} \sim \pi_{\theta}\left(\cdot \mid s^{\prime}\right)
\]
\FOR{$i \in K$}  
\STATE Update $\phi_i$ with gradient descent using
\[
\nabla_{\phi} \frac{1}{|B|} \sum_{\left(s, a, r, s^{\prime} \right) \in B}\left(Q_{\phi_{i}}(s, a)-y\right)^{2} 
\]
\STATE Update target networks with $\phi_{\operatorname{targ},i} \leftarrow \rho \phi_{\operatorname{targ},i}+(1-\rho) \phi_{i}$
\ENDFOR
\ENDFOR 
\STATE Update policy parameters $\theta$ with gradient ascent using
\[
\nabla_{\theta} \frac{1}{|B|} \sum_{s \in B}\left(\frac{1}{N}\sum_{i =1}^{N} Q_{\phi_{i}}\left(s, \tilde{a}_{\theta}(s)\right)-\alpha \log \pi_{\theta}\left(\tilde{a}_{\theta}(s) | s\right)\right),
\quad \tilde{a}_{\theta}(s) \sim \pi_{\theta}(\cdot \mid s)
\]
\UNTIL{}
\end{algorithmic}
\end{algorithm}

\section{Proofs}
\subsection{Proof of Theorem 1}
\begin{lemma}
Let $X_i\sim$U$\left(a,b\right)$ and $Y_i\sim$B$\left(1,p_i\right)$. Then for
$Z_i\,=\,X_i+cY_i(d-X_i)$, $Z_{ij\text{min}}=\,min(Z_i,Z_j)$, $d\in(a,b)$, $c\in(0,1)$ we have:
\begin{enumerate}[(i)]
\item 
$Ee^{Z_it}=(1-p_i)(\frac{e^{tb}-e^{ta}}{t(b-a)})+p_ie^{cdt}(
\frac{e^{(1-c)tb}-e^{(1-c)ta}}{t(1-c)(b-a)})$ and $E[Z_i] =
1-p_i)(\frac{e^{tb}-e^{ta}}{t(b-a)})+p_ie^{cdt}(
\frac{e^{(1-c)tb}-e^{(1-c)ta}}{t(1-c)(b-a)})$
\item The distributions of $Z_i$ and $Z_{ij\text{min}}$ are characterized by:
\begin{align}
F_{Z_i}(z)&=\frac{(1-p_i)}{a-b}((z-b)\boldsymbol{1}_{(b,\infty)}(z)-(z-a)
\boldsymbol{1}_{(a,\infty)}(z))\\
&+\frac{p_i}{(1-c)(a-b)}((z-(dc-b(c-1)))\boldsymbol{1}_{(dc-b(c-1),
\infty)}(z)\\
&-(z-(dc-a(c-1)))\boldsymbol{1}_{(dc-a(c-1),\infty)}(z))
\end{align}
and
\begin{align}
F_{Z_{ij\text{min}}}(z)&=\,\frac{2(1-p_i)}{a-b}\beta(z,b,a)+
\frac{2p_i}{(1-c)(a-b)}\beta(z,dc-b(c-1),dc-a(c-1))\\
&-(\frac{(1-p_{i|j})}{a-b}\beta(z,b,a)\\
&+\frac{p_{i|j}}{(1-c)(a-b)}\beta(z,dc-b(c-1),dc-a(c-1)))(\frac{(1-p_i)}{a-b}\beta(z,b,a)\\
&+\frac{p_i}{(1-c)(a-b)}\beta(z,dc-b(c-1),dc-a(c-1))
\end{align}
respectively, where 
\[p_i\,=\,P(Y_i=1) \text{ and } p_{i|j}=P(Y_i=1|Y_j=1),
\]
\[
\beta(z,\theta,\alpha)\,=\,((z-\theta)\boldsymbol{1}_{(\theta,
\infty)}(z)-(z-\alpha)\boldsymbol{1}_{(\alpha,\infty)}(z)),\,\theta\geq
\alpha
\]
\end{enumerate}
\end{lemma}

\begin{proof}
\begin{enumerate}[(i)]
\item The moment generating function of $Z_i\,=\,X_i+cY_i(d-X_i)$
\begin{align}
Ee^{Z_it}\,&=\,Ee^{(X_i+cY_i(d-X_i))t}=E[Ee^{(X_i+cY_i(d-X_i))t}|X_i] \nonumber \\
&=E[e^{X_it}E[e^{Y_ic(d-X)t}|X]]=E[e^{X_it}[1-p_i+p_ie^{c(d-X_i)t}]]=(1-p_i)E[
e^{X_it}]+p_ie^{cdt}e^{X(1-c)t} \nonumber \\
&=\,(1-p_i)(\frac{e^{tb}-e^{ta}}{t(b-a)})+p_ie^{cdt}(
\frac{e^{(1-c)tb}-e^{(1-c)ta}}{t(1-c)(b-a)}).
\end{align}
\item It follows from (i) that 

$\mathcal{L}\{f\}(t)\,=\,Ee^{-Z_it}\,=\,(1-p_i)(%
\frac{e^{-tb}-e^{-ta}}{t(a-b)})+p_ie^{-cdt}(%
\frac{e^{(c-1)tb}-e^{(c-1)ta}}{t(1-c)(a-b)})$

and
\begin{align}
F_{Z_i}(z)\,&=\mathcal{L}^{-1}\{\frac{1}{t}\mathcal{L}\{f\}(t)\,\}(z)\,\nonumber \\
&=\,\mathcal{L}^{-1}\{\,(1-p_i)(%
\frac{e^{-tb}-e^{-ta}}{t^2(a-b)})+p_ie^{-cdt}(%
\frac{e^{(c-1)tb}-e^{(c-1)ta}}{t^2(1-c)(a-b)})\,\}(z)\, \nonumber\\
&=\mathcal{L}^{-1}\{\,\frac{(1-p_i)}{a-b}(\frac{e^{-tb}-e^{-ta}}{t^2})+
\frac{p_i}{(1-c)(a-b)}e^{-cdt}(\frac{e^{(c-1)tb}-e^{(c-1)ta}}{t^2})\,\}(z) \nonumber \\
&=\frac{(1-p_i)}{a-b}\mathcal{L}^{-1}\{\,(\frac{e^{-tb}-e^{-ta}}{t^2})\}(z)+
\frac{p_i}{(1-c)(a-b)}\mathcal{L}^{-1}\{\,e^{-cdt}(
\frac{e^{(c-1)tb}-e^{(c-1)ta}}{t^2})\,\}(z) \nonumber\\
&=\frac{(1-p_i)}{a-b}(\mathcal{L}^{-1}\{\,\frac{e^{-tb}}{t^2})\}(z)-
\mathcal{L}^{-1}\{\,\frac{e^{-ta}}{t^2})\}(z)) \nonumber \\
&+\frac{p_i}{(1-c)(a-b)}(\mathcal{L}^{-1}\{\,
\frac{e^{-t(dc-b(c-1))}}{t^2})\}(z)-\mathcal{L}^{-1}\{\,
\frac{e^{-t(dc-a(c-1)}}{t^2})\}(z)) \nonumber \\
&=\frac{(1-p_i)}{a-b}((z-b)\boldsymbol{1}_{(z>b)}-(z-a)
\boldsymbol{1}_{(z>a)}) \nonumber \\
&+\frac{p_i}{(1-c)(a-b)}((z-(dc-b(c-1)))
\boldsymbol{1}_{(z>(dc-b(c-1)))}-(z-(dc-a(c-1)))
\boldsymbol{1}_{(z>(dc-a(c-1)))}) \nonumber \\
&=\frac{(1-p_i)}{a-b}((z-b)\boldsymbol{1}_{(b,\infty)}(z)-(z-a)
\boldsymbol{1}_{(a,\infty)}(z)) \nonumber \\
&+\frac{p_i}{(1-c)(a-b)}((z-(dc-b(c-1)))\boldsymbol{1}_{(dc-b(c-1),
\infty)}(z)-(z-(dc-a(c-1)))\boldsymbol{1}_{(dc-a(c-1),\infty)}(z))
\end{align}

Notice that
\begin{align*}
E[Z_i]\,&=\,(1-p_i)(\frac{a+b}{2})\,+\,\frac{p_i}{(1-c)(b-a)}[
\frac{1}{2}(c-1)(b-a)[(c-1)(a+b)-2cd]]\\
&=\,(1-p_i)(\frac{a+b}{2})\,+\,\frac{p_i}{2}[(1-c)(a+b)+2cd] \\
&=\,(1-cp_i)(\frac{a+b}{2})\,+cdp_i.
\end{align*}

Furthermore we can derive the variance Var$(Z_i)\,=\,E[Z^2_i]-(E[Z_i])^2$ using
\[
(E[Z_i])^2=(1-cp_i)^2(\frac{a+b}{2})^2\,+(1-cp_i)(a+b)cdp_i+(cdp_i)^2
\]
and 
\begin{align*}
E[Z^2_i]&=(1-p_i)(\frac{a^2+ab+b^2}{3})\,+\,
\frac{p_i}{3}[(1-c)^2(a^2+ab+b^2)+3cd((a+b)(1-c)+cd)] \\
&=(1-p_i)(\frac{a^2+ab+b^2}{3})\,+\,
\frac{p_i}{3}[(1-c)^2(a^2+ab+b^2)+3cd((a+b)(1-c)+cd)] \\
&=(1-2cp_i+p_ic^2)(\frac{a^2+ab+b^2}{3})+cdp_i((a+b)(1-c)+cd)
\end{align*}
\item

Since $Z_i$ are identically distributed but not necessarily
independent$,\,$the distribution of \ $\,Z_{ijmin}=\,min(Z_i,Z_j)\,$is characterized by

\[
F_{Z_{ijmin}}(z\,)\,=F_{Z_i}(z)+F_{Z_j}(z)-F_{Z_iZ_j}(z,z)\,=\,
2F_{Z_i}(z)-F_{Z_iZ_j}(z,z)
\]
where
\[
F_{Z_iZ_j}(z,z)\,=P(Z_j\leq z,Z_j\leq z)=F_{Z_i|Z_j}(z)F_{Z_j}(z)
\] is the joint distribution of $Z_i$ and $Z_j$. 
Hence,
\begin{align*}
F_{Z_{ijmin}}(z\,)&=\frac{2(1-p_i)}{a-b}((z-b)\boldsymbol{1}_{(b,%
\infty)}(z)-(z-a)\boldsymbol{1}_{(a,\infty)}(z)) \\
&+\frac{2p_i}{(1-c)(a-b)}((z-(dc-b(c-1)))\boldsymbol{1}_{(dc-b(c-1),\infty)}(z)-(z-(dc-a(c-1)))\boldsymbol{1}_{(dc-a(c-1),\infty)}(z))\\
&-(\frac{(1-p_{i|j})}{a-b}((z-b)\boldsymbol{1}_{(b,\infty)}(z)-(z-a)\boldsymbol{1}_{(a,\infty)}(z)) \\
& +\frac{p_{i|j}}{(1-c)(a-b)}((z-(dc-b(c-1)))\boldsymbol{1}_{(dc-b(c-1),
\infty)}(z)\\
&-(z-(dc-a(c-1)))\boldsymbol{1}_{(dc-a(c-1),\infty)}(z))
)(\frac{(1-p_i)}{a-b}((z-b)\boldsymbol{1}_{(b,\infty)}(z)-(z-a)%
\boldsymbol{1}_{(a,\infty)}(z))\\
&+\frac{p_i}{(1-c)(a-b)}((z-(dc-b(c-1)))\boldsymbol{1}_{(dc-b(c-1),
\infty)}(z)-(z-(dc-a(c-1)))\boldsymbol{1}_{(dc-a(c-1),\infty)}(z))\\
&=\frac{2(1-p_i)}{a-b}\beta(z,b,a)+\frac{2p_i}{(1-c)(a-b)}\beta(z,dc-b(c-1),%
dc-a(c-1))-(\frac{(1-p_{i|j})}{a-b}\beta(z,b,a)\\
&+\frac{p_{i|j}}{(1-c)(a-b)}\beta(z,dc-b(c-1),dc-a(c-1))
)(\frac{(1-p_i)}{a-b}\beta(z,b,a) \\
&+\frac{p_i}{(1-c)(a-b)}\beta(z,dc-b(c-1),dc-a(c-1))
\end{align*}
where,
\[
p_i\,=\,P(Y_i=1)=p_j,\,p_{i|j}=P(Y_i=1|Y_j=1),\,
\beta(z,\theta,\alpha)\,=\,((z-\theta)\boldsymbol{1}_{(\theta,
\infty)}(z)-(z-\alpha)\boldsymbol{1}_{(\alpha,\infty)}(z)),\,\theta\geq
\alpha.
\]
Notice that
\begin{align*}
\frac{d}{dz}[\beta(z,\theta,\alpha)]&=(z-\theta)\delta(z-\theta)+
\boldsymbol{1}_{(\theta,\infty)}(z)-(z-\alpha)\delta(z-\alpha)-
\boldsymbol{1}_{(\alpha,\infty)}(z)\\
&=(z-\theta)\delta(z-\theta)-(z-\alpha)\delta(z-\alpha)+
\boldsymbol{1}_{(\theta,\alpha)}(z)=\boldsymbol{1}_{(\theta,\alpha)}(z)
\end{align*}
where $\delta(x)$ is the Dirac delta function, implying that
\begin{align*}
f_{Z_{ijmin}}(z)\,&=\,\frac{d}{dz}[F_{Z_{ijmin}}(z)]\\
&=\frac{2(1-p_i)}{a-b}\boldsymbol{1}_{(a,b)}(z)+\frac{2p_i}{(1-c)(a-b)}\boldsymbol{1}_{(dc-b(c-1),dc-a(c-1))}(z)\\
&-(\frac{(1-p_{i|j})}{a-b}\beta(z,b,a)+\frac{p_{i|j}}{(1-c)(a-b)}\beta(z,dc-b(c-1),dc-a(c-1)))(\frac{(1-p_i)}{a-b}\boldsymbol{1}_{(a,b)}(z)\\
&+\frac{p_i}{(1-c)(a-b)}\boldsymbol{1}_{(dc-b(c-1),dc-a(c-1))}(z))\\
&-(\frac{(1-p_{i|j})}{a-b}\boldsymbol{1}_{(a,b)}(z)+\frac{p_{i|j}}{(1-c)(a-b)}
\boldsymbol{1}_{(dc-b(c-1),dc-a(c-1))}(z))(\frac{(1-p_i)}{a-b}\beta(z,b,a)\\
&+\frac{p_i}{(1-c)(a-b)}\beta(z,dc-b(c-1),dc-a(c-1))\\
&=\frac{2(1-p_i)}{a-b}\boldsymbol{1}_{(a,b)}(z)+\frac{2p_i}{(1-c)(a-b)}\boldsymbol{1}_{(dc-b(c-1),dc-a(c-1))}(z)\\
&-\frac{(1-p_{i|j})}{a-b}\beta(z,b,a)\frac{(1-p_i)}{a-b}\boldsymbol{1}_{(a,
b)}(z)-\frac{(1-p_{i|j})}{a-b}\beta(z,b,a)\frac{p_i}{(1-c)(a-b)}%
\boldsymbol{1}_{(dc-b(c-1),dc-a(c-1))}(z)\\
&-\frac{p_{i|j}}{(1-c)(a-b)}\beta(z,dc-b(c-1),dc-a(c-1))\frac{(1-p_i)}{a-b}
\boldsymbol{1}_{(a,b)}(z)\\
&-\frac{p_{i|j}}{(1-c)(a-b)}\beta(z,dc-b(c-1),dc-a(c-1))\frac{p_i}{(1-c)(a-b)}
\boldsymbol{1}_{(dc-b(c-1),dc-a(c-1))}(z)\\
&-\frac{(1-p_{i|j})}{a-b}\boldsymbol{1}_{(a,b)}(z)\frac{(1-p_i)}{a-b}\beta(z,b,
a)-\frac{(1-p_{i|j})}{a-b}\boldsymbol{1}_{(a,b)}(z)\frac{p_i}{(1-c)(a-b)}
\beta(z,dc-b(c-1),dc-a(c-1))\\
&-\frac{p_{i|j}}{(1-c)(a-b)}\boldsymbol{1}_{(dc-b(c-1),dc-a(c-1))}(z)
\frac{(1-p_i)}{a-b}\beta(z,b,a)\\
&-\frac{p_{i|j}}{(1-c)(a-b)}\boldsymbol{1}_{(dc-b(c-1),dc-a(c-1))}(z)
\frac{p_i}{(1-c)(a-b)}\beta(z,dc-b(c-1),dc-a(c-1))\\
&=\frac{2(1-p_i)}{a-b}\boldsymbol{1}_{(a,b)}(z)+\frac{2p_i}{(1-c)(a-b)}
\boldsymbol{1}_{(dc-b(c-1),dc-a(c-1))}(z)\\
&-\frac{(1-p_{i|j})}{a-b}(a-b)\frac{(1-p_i)}{a-b}\boldsymbol{1}_{(a,b)}(z)-
\frac{(1-p_i)}{a-b}(a-z)\frac{p_i}{(1-c)(a-b)}\boldsymbol{1}_{(dc-b(c-1),dc-a(c-1))}(z)\\
&-\frac{p_{i|j}}{(1-c)(a-b)}(dc-a(c-1)-z)\frac{(1-p_i)}{a-b}\boldsymbol{1}_{(a,b)}(z)\\
&-\frac{p_{i|j}}{(1-c)(a-b)}(c-1)(b-a)\frac{p_i}{(1-c)(a-b)}
\boldsymbol{1}_{(dc-b(c-1),dc-a(c-1))}(z)\\
&-\frac{(1-p_{i|j})}{a-b}\boldsymbol{1}_{(a,b)}(z)\frac{(1-p_i)}{a-b}(a-b)-
\frac{(1-p_{i|j})}{a-b}\boldsymbol{1}_{(a,b)}(z)
\frac{p_i}{(1-c)(a-b)}(dc-a(c-1)-z)\\
&-\frac{p_{i|j}}{(1-c)(a-b)}\boldsymbol{1}_{(dc-b(c-1),dc-a(c-1))}(z)\frac{(1-p_i)}{a-b}(a-z)\\
&-\frac{p_{i|j}}{(1-c)(a-b)}\boldsymbol{1}_{(dc-b(c-1),dc-a(c-1))}(z)\frac{p_i}{(1-c)(a-b)}(c-1)(b-a)\\
&=\frac{2(1-p_i)}{a-b}\boldsymbol{1}_{(a,b)}(z)+\frac{2p_i)}{(1-c)(a-b)}\boldsymbol{1}_{(dc-b(c-1),dc-a(c-1))}(z)\\
&-\frac{(1-p_{i|j})}{a-b}\frac{(1-p_i)}{a-b}(a-b)\boldsymbol{1}_{(a,b)}(z)-\frac{(1-p_{i|j})}{a-b}\frac{p_i}{(1-c)(a-b)}(a-z)\boldsymbol{1}_{(dc-b(c-1),dc-a(c-1))}(z)\\
&-\frac{p_{i|j}}{(1-c)(a-b)}\frac{(1-p_i)}{a-b}(dc-a(c-1)-z)\boldsymbol{1}_{(a,b)}(z)\\
&-\frac{p_{i|j}}{(1-c)(a-b)}(c-1)(b-a)\frac{p_i}{(1-c)(a-b)}%
\boldsymbol{1}_{(dc-b(c-1),dc-a(c-1))}(z)\\
&-\frac{(1-p_{i|j})}{a-b}\frac{(1-p_i)}{a-b}(a-b)\boldsymbol{1}_{(a,b)}(z)-
\frac{(1-p_{i|j})}{a-b}\frac{p_i}{(1-c)(a-b)}(dc-a(c-1)-z)\boldsymbol{1}_{(a,
b)}(z)\\
&-\frac{p_{i|j}}{(1-c)(a-b)}\frac{(1-p_i)}{a-b}(a-z)\boldsymbol{1}_{(dc-b(c-1),
dc-a(c-1))}(z)\\
&-\frac{p_{i|j}}{(1-c)(a-b)}\frac{p_i}{(1-c)(a-b)}(c-1)(b-a)%
\boldsymbol{1}_{(dc-b(c-1),dc-a(c-1))}(z)
\end{align*}

From this we arrive at 
\begin{align*}
EZ_{ijmin}^2&=\int_{-\infty}^\infty x^2f_{Z_{ijmin}}(x)dx\\
&=\frac{2(1-p_i)}{a-b}\frac{x^3}{3}|^b_a+\frac{2p_i}{(1-c)(a-b)}
\frac{x^3}{3}|^{dc-a(c-1)}_{dc-b(c-1)}\\
&-\frac{(1-p_{i|j})}{a-b}\frac{(1-p_i)}{a-b}(a-b)\frac{x^3}{3}|^b_a-
\frac{(1-p_{i|j})}{a-b}\frac{p_i}{(1-c)(a-b)}(a
\frac{x^3}{3}|^{dc-a(c-1)}_{dc-b(c-1)}-
\frac{x^4}{4}|^{dc-a(c-1)}_{dc-b(c-1)})\\
&-\frac{p_{i|j}}{(1-c)(a-b)}\frac{(1-p_i)}{a-b}((dc-a(c-1))\frac{x^3}{3}|^b_a-
\frac{x^4}{4}|^b_a)\\
&-\frac{p_{i|j}}{(1-c)(a-b)}(c-1)(b-a)\frac{p_i}{(1-c)(a-b)}
\frac{x^3}{3}|^{dc-a(c-1)}_{dc-b(c-1)}\\
&-\frac{(1-p_{i|j})}{a-b}\frac{(1-p_i)}{a-b}(a-b)\frac{x^3}{3}|^b_a-
\frac{(1-p_{i|j})}{a-b}\frac{p_i}{(1-c)(a-b)}((dc-a(c-1))\frac{x^3}{3}|^b_a-
\frac{x^4}{4}|^b_a)\\
&-\frac{p_{i|j}}{(1-c)(a-b)}\frac{(1-p_i)}{a-b}(a
\frac{x^3}{3}|^{dc-a(c-1)}_{dc-b(c-1)}-
\frac{x^4}{4}|^{dc-a(c-1)}_{dc-b(c-1)})\\
&-\frac{p_{i|j}}{(1-c)(a-b)}\frac{p_i}{(1-c)(a-b)}(c-1)(b-a)
\frac{x^3}{3}|^{dc-a(c-1)}_{dc-b(c-1)}
\end{align*}
and
\begin{align*}
EZ_{ijmin}&=\int_{-\infty}^\infty xf_{Z_{ijmin}}(x)dx
=\frac{2(1-p_i)}{a-b}\frac{x^2}{2}|^b_a+\frac{2p_i}{(1-c)(a-b)}
\frac{x^2}{2}|^{dc-a(c-1)}_{dc-b(c-1)}\\
&-\frac{(1-p_{i|j})}{a-b}\frac{(1-p_i)}{a-b}(a-b)\frac{x^2}{2}|^b_a-
\frac{(1-p_{i|j})}{a-b}\frac{p_i}{(1-c)(a-b)}(a
\frac{x^2}{2}|^{dc-a(c-1)}_{dc-b(c-1)}\\
&-\frac{x^3}{3}|^{dc-a(c-1)}_{dc-b(c-1)})\\
&-\frac{p_{i|j}}{(1-c)(a-b)}\frac{(1-p_i)}{a-b}((dc-a(c-1))\frac{x^2}{2}|^b_a-
\frac{x^3}{3}|^b_a)\\
&-\frac{p_{i|j}}{(1-c)(a-b)}(c-1)(b-a)\frac{p_i}{(1-c)(a-b)}
\frac{x^2}{2}|^{dc-a(c-1)}_{dc-b(c-1)}\\
&-\frac{(1-p_{i|j})}{a-b}\frac{(1-p_i)}{a-b}(a-b)\frac{x^2}{2}|^b_a-
\frac{(1-p_{i|j})}{a-b}\frac{p_i}{(1-c)(a-b)}((dc-a(c-1))\frac{x^2}{2}|^b_a-
\frac{x^3}{3}|^b_a)\\
&-\frac{p_{i|j}}{(1-c)(a-b)}\frac{(1-p_i)}{a-b}(a
\frac{x^2}{2}|^{dc-a(c-1)}_{dc-b(c-1)}-
\frac{x^3}{3}|^{dc-a(c-1)}_{dc-b(c-1)})\\
&-\frac{p_{i|j}}{(1-c)(a-b)}\frac{p_i}{(1-c)(a-b)}(c-1)(b-a)
\frac{x^2}{2}|^{dc-a(c-1)}_{dc-b(c-1)}
\end{align*}
\end{enumerate}

\end{proof}

Next we prove \Cref{theorem:Theorem 1}.
\begin{proof}
Recall that
\begin{equation*}
Q_{t+1}^i(s,a)\,=Q_t^\pi(s,a)+(\varepsilon^i_t(s,a)+\alpha I_i[Y^{MQ}-Q_t^\pi(s,a)-\varepsilon^i_t(s,a)]
\end{equation*}

where $\varepsilon^i_t(s,a)\,$identically distributed
$U(-\tau$,$\tau),\,I_i\sim$Bernoulli$(p_i)\,$ .
Hence,
\begin{multline*}
VarQ^{avg}=Var(\frac{1}{M}\sum_{i\in
\cal{M}}Q_{t+1}^i(s,a)\,|Y^{MQ})\,=\frac{1}{M^2}Var(\sum_{i\in
\cal{M}}Q_{t+1}^i(s,a)\,|Y^{MQ})\\
=\frac{1}{M^2}[\sum_{i\in \cal{M}}Var(Q_{t+1}^i(s,a)\,|Y^{MQ})+\sum_{i\neq
j}Cov(Q_{t+1}^i(s,a)\,|Y^{MQ},Q_{t+1}^j(s,a)\,|Y^{MQ})]\\
=\frac{1}{M^2}[\sum_{i\in \cal{M}}Var(Q_{t+1}^i(s,a)\,|Y^{MQ})+\sum_{i\neq
j}(E[Q_{t+1}^i(s,a)\,|Y^{MQ}]E[Q_{t+1}^j(s,a)\,|Y^{MQ}]-E[Q_{t+1}^i(s,a)\,
|Y^{MQ}]E[Q_{t+1}^j(s,a)\,|Y^{MQ}])]
\end{multline*}
Consider $|\cal{M}|=2$. Then
\begin{multline*}
VarQ^{avg}=\frac{1}{4}[Var(Q_{t+1}^i(s,a)\,|Y^{MQ})+Var(Q_{t+1}^j(s,a)
|Y^{MQ})+E[Q_{t+1}^i(s,a)\,|Y^{MQ}]E[Q_{t+1}^j(s,a)\,|Y^{MQ}]-E[Q_{t+1}^i(s,
a)\,|Y^{MQ}]E[Q_{t+1}^j(s,a)\,|Y^{MQ}]]\\
=\frac{1}{4}[2Var(Q_{t+1}^i(s,a)\,|Y^{MQ})+E[Q_{t+1}^i(s,a)\,
|Y^{MQ}]E[Q_{t+1}^j(s,a)\,|Y^{MQ}]-E[Q_{t+1}^i(s,a)\,|Y^{MQ}]E[Q_{t+1}^j(s,a)
\,|Y^{MQ}]]
\end{multline*}
Thus by Lemma 1,
\begin{align*}
VarQ^{avg}&=\frac{1}{4}[2Var(Q_{t+1}^i(s,a)\,|Y^{MQ})+E[Q_{t+1}^i(s,a)\,
|Y^{MQ}]E[Q_{t+1}^j(s,a)\,|Y^{MQ}]-E[Q_{t+1}^i(s,a)\,|Y^{MQ}]E[Q_{t+1}^j(s,a)
\,|Y^{MQ}]]\\
&=\frac{1}{4}[2Var(Q_{t+1}^i(s,a)\,|Y^{MQ})+(\alpha(Y^{MQ}-Q_t^\pi(s,
a)))^2(p_ip_{j|i}-p_ip_j)]\\
&=\frac{1}{4}[2Var(Q_{t+1}^i(s,a)\,|Y^{MQ})+(\alpha(Y^{MQ}-Q_t^\pi(s,
a)))^2(p_{ij}-p_i^2)]\\
&=\psi+\varphi(p_ip_{j|i}-p_i^2)
\end{align*}
where $\psi=\,\frac{1}{2}Var(Q_{t+1}^i(s,a)\,|Y^{MQ}),\varphi=\frac{1}{4}
\alpha$($Y^{MQ}-Q_t^\pi(s,a)$)$^2$\\
Notice that $\psi,\varphi\geq 0.\,$
So $VarQ^{avg}$ breaks down to
\begin{equation*}
VarQ^{avg}=
\begin{cases}
&\psi\text{, if none of }Q_t^i(s,a),Q_{t+1}^i(s,a)\text{ were
updated}\,\text{or if they were updated by random sampling }\\
&\psi+\varphi(p_{ij}-p_i^2)\text{, if both }Q_t^i(s,a)\text{ and
}Q_{t+1}^i(s,a)\text{ }\,\text{were updated}\,\text{according to }k\text{-DPP}
\end{cases}
\end{equation*}
Since $k$-DPP is a repulsive process,$\,$if $Q_t^i(s,a),Q_{t+1}^i(s,a)$ \
$\,$are close $p_{ij}<p_ip_j=p^2_i$, and$\,\,$as they get further apart,
because of our choice of kernel based on CKA,
$p_{ij}\rightarrow p_ip_j$ and $p_i\rightarrow\frac{1}{N}\,$ such
that even when points are farther apart, variance is still reduced. Proof for general $M$
follows by induction$\,$and similar process can be followed to prove case for
$VarQ^{min}$.

\end{proof}

\end{document}